\documentclass{article}

 \usepackage[preprint]{neurips_2026}

\usepackage[utf8]{inputenc} %
\usepackage[T1]{fontenc}    %
\usepackage{url}            %
\usepackage{booktabs}       %
\usepackage{amsfonts}       %
\usepackage{nicefrac}       %
\usepackage{microtype}      %
\usepackage{xcolor}         %

\title{Back to Blackwell: Closing the Loop on Intransitivity in Multi-Objective Preference Fine-Tuning}

\author{%
  Jiahao Zhang\thanks{Equal contribution.} \\
  Carnegie Mellon University\\
  Pittsburgh, PA 15213 \\
  \texttt{jiahaozhang@cmu.edu} \\
  \And
  Lujing Zhang\footnotemark[1] \\
  Carnegie Mellon University\\
  Pittsburgh, PA 15213 \\
  \texttt{lujingz@andrew.cmu.edu}
  \And
  Keltin Grimes \\
  Carnegie Mellon University\\
  Pittsburgh, PA 15213 \\
  \texttt{kgrimes@andrew.cmu.edu}
  \AND
  Zhuohao Yu \\
  Carnegie Mellon University\\
  Pittsburgh, PA 15213 \\
  \texttt{zhuohaoy@andrew.cmu.edu}
  \And
  Gokul Swamy \\
  Carnegie Mellon University\\
  Pittsburgh, PA 15213 \\
  \texttt{gswamy@cmu.edu}
  \And
  Zhiwei Steven Wu \\
  Carnegie Mellon University\\
  Pittsburgh, PA 15213 \\
  \texttt{zstevenwu@cmu.edu}
}

\usepackage[T1]{fontenc}
\usepackage{algorithm,algpseudocode} 
\usepackage{tikz}
\usetikzlibrary{backgrounds}
\usepackage{enumitem}
\usetikzlibrary{matrix} 
\usepackage{xurl}
\usepackage{microtype}
\usepackage{graphicx}
\usepackage{subcaption}
\usepackage{booktabs} %

\usepackage{amsmath,amsfonts,bm}

\def\eqref#1{equation~\ref{#1}}

\def\1{\bm{1}}

\DeclareMathAlphabet{\mathsfit}{\encodingdefault}{\sfdefault}{m}{sl}
\SetMathAlphabet{\mathsfit}{bold}{\encodingdefault}{\sfdefault}{bx}{n}

\newcommand{\E}{\mathbb{E}}

\newcommand{\KL}{\mathbb{D}_{\mathrm{KL}}}

\newcommand{\Cs}{\mathcal{C}}

\newcommand{\Ps}{\mathcal{P}}

\newcommand{\Xs}{\mathcal{X}}
\newcommand{\Ys}{\mathcal{Y}}

\newcommand{\reff}{\mathsf{ref}}

\DeclareMathOperator*{\argmax}{arg\,max}
\DeclareMathOperator*{\argmin}{arg\,min}

\newcommand{\notshow}[1]{{}}
\newcommand{\AutoAdjust}[3]{{ \mathchoice{ \left #1 #2  \right #3}{#1 #2 #3}{#1 #2 #3}{#1 #2 #3} }}
\newcommand{\Xcomment}[1]{{}}

\newcommand{\InParentheses}[1]{\AutoAdjust{(}{#1}{)}}
\newcommand{\InBrackets}[1]{\AutoAdjust{[}{#1}{]}}
\newcommand{\InAngles}[1]{\AutoAdjust{\langle}{#1}{\rangle}}

\newcommand{\dist}{\operatorname{dist}}

\newcommand{\nogame}{\texttt{PROSPER-VB}}
\newcommand{\PROSPER}{\texttt{PROSPER}}
\newcommand{\fullcheck}{\texttt{PROSPER-JC}}

\usepackage[colorlinks=true,allcolors=perfblue]{hyperref} 
\usepackage[customcolors]{hf-tikz}
\definecolor{expert}{HTML}{008000}
\definecolor{error}{HTML}{f96565}
\definecolor{learner}{HTML}{F79646}
\definecolor{perfblue}{RGB}{64, 114, 175}

\usepackage{pifont}
\usepackage{float}

\usepackage{amsfonts, amsmath, amssymb, bm}       %
\usepackage{amsthm}
\usepackage{thmtools}

\usepackage{amsmath}
\usepackage{amssymb}
\usepackage{mathtools}
\usepackage{amsthm}

\usepackage[capitalize,noabbrev]{cleveref}
\crefname{appendix}{appendix}{appendices}
\Crefname{appendix}{Appendix}{Appendices}

\theoremstyle{plain}
\newtheorem{theorem}{Theorem}[section]

\newtheorem{lemma}[theorem]{Lemma}

\theoremstyle{definition}
\newtheorem{definition}[theorem]{Definition}
\newtheorem{assumption}[theorem]{Assumption}
\theoremstyle{remark}

\usepackage{thmtools}
\usepackage{thm-restate}

\usepackage[textsize=tiny]{todonotes}

\begin{document}
\maketitle
\begin{abstract}
A recurring challenge in preference fine-tuning (PFT) is handling \textit{intransitive} (i.e., cyclic) preferences. Intransitive preferences often stem from either \textit{(i)} inconsistent rankings along a single objective or \textit{(ii)} scalarizing multiple objectives into a single metric. Regardless of their source, the downstream implication of intransitive preferences is the same: there is no well-defined optimal policy, breaking a core assumption of the standard PFT pipeline. In response, we propose a novel, game-theoretic solution concept, the \textit{Maximum Entropy Blackwell Winner} (\textit{MaxEntBW}), that is well-defined under multi-objective intransitive preferences. To enable computing MaxEntBWs at scale, we derive $\texttt{PROSPER}$: a provably efficient PFT algorithm. Unlike prior self-play techniques, $\texttt{PROSPER}$ directly handles multiple objectives without requiring scalarization. We then apply $\texttt{PROSPER}$ to the problem of fine-tuning large language models (LLMs) from multi-objective LLM-as-a-Judge feedback (e.g., rubric-based judges), a setting where both sources of intransitivity arise. We find that $\texttt{PROSPER}$ outperforms all baselines considered across both instruction following and general chat benchmarks, releasing trained model checkpoints at the \href{https://huggingface.co/MisDrifter/Qwen2.5-7B-PROSPER}{7B} and \href{https://huggingface.co/MisDrifter/Qwen2.5-3B-PROSPER}{3B} parameter scales.
\end{abstract}

\section{Introduction}\label{sec:introduction}

Author Ted Chiang famously described large language models (LLMs) as a ``\textit{blurry JPEG of the web}'' \citep{chiang2023blurry}. One way of interpreting his words is that after being trained on vast swathes of internet text, the choices of an LLM ``judge'' \citep{zheng2023judging} are often inconsistent. Thus, when LLMs are used to rank a set of options  $\{A, B, C\}$, they often exhibit \textit{intransitive} (i.e., cyclic) preferences where they rank $A \succ B$ and $B \succ C$, but $A \not\succ C$ \citep{xu2025investigatingnontransitivityllmasajudge}.

By construction, there is no \textit{total ordering} over options under intransitive preferences, as \textit{every} option loses to another. Thus, the standard \textit{preference fine-tuning} (PFT) pipeline \citep{stiennon2020learning, ouyang2022training, song2024importance, swamy2026all}, in which one attempts to learn such a total ordering via training a reward model, is fundamentally unequipped to robustly handle the intransitive preferences we see when attempting to learn from noisy LLM judge feedback \citep{procaccia2006distortion, boutilier2012optimal, sorensen2024roadmap, golz2025distortion}.

In response, prior work has primarily explored two mitigatory strategies. First, to reduce the baseline inconsistency of LLM judge feedback, recent work has proposed to ask the LLM judge to explicitly evaluate candidate responses along multiple objectives, rather than making a single aggregate judgment \citep{viswanathan2025checklists, gunjal2025rubrics}. This echoes observations from psychology, where one of the most common root causes of intransitivity is \textit{scalarization} of multiple objectives into a single metric \citep{tversky1969intransitivity}.

Second, to tackle the remaining per-objective inconsistency, a variety of authors have proposed to frame PFT as zero-sum game solving \citep{munos2023nash,swamy2024minimaximalist, rosset2024direct}, building on ideas from social choice theory on how to robustly handle intransitive preferences \citep{kreweras1965aggregation, fishburn1984probabilistic}. Intuitively, these approaches attempt to find a policy that the LLM judge prefers as often as possible against the worst-case comparator -- a well-defined notion even under intransitive preferences.

Given the complementary nature of these two lines of work, one may naturally hope to combine them into a single, gestalt solution. However doing so is technically challenging for a rather subtle reason: when one splits an aggregate judgement into multiple objectives, the Minimax Theorem \citep{vonneumann1928} no longer holds, as famously observed by \citet{Blackwell1956Approachability}. The downstream implication of this fact is that standard game-solving machinery (e.g., in the proofs of \citet{swamy2024minimaximalist}) breaks down, making it unclear how to directly extend prior algorithms that build on said machinery to the multi-objective PFT setting.

Faced with this roadblock, we circle back and ask a first-principles question about the multi-objective PFT setting: \textit{what should it mean to do well when evaluated under multiple objectives, each of which may be globally inconsistent?}

Our answer is the \textbf{\textit{Maximum Entropy Blackwell Winner (MaxEntBW):}} \textit{a policy that is preferred to all other policies within its local neighborhood, regardless of the criteria it is judged under.} Remarkably, even though it is a strictly more general solution concept than those explored in prior work, the underlying structure of MaxEntBWs means one can efficiently compute them without requiring the Minimax Theorem to hold.
Although MaxEntBWs are defined via a game between two LLM players, their structure permits an efficient reduction to \textit{\textbf{a single-player optimization}} problem: only the learner’s strategy must be maintained and updated, and the update can be implemented via square-loss regression. Thus, we can capitalize on the more accurate feedback we get from LLM judges by asking them to explicitly evaluate responses along multiple objectives without needing to complicate the policy optimization pipeline.
We call our algorithm \texttt{PROSPER}, for \textbf{PR}eference-based \textbf{O}ptimization with a \textbf{S}ingle \textbf{P}layer over \textbf{E}ntire \textbf{R}ubrics. Explicitly, our contributions are threefold:

\begin{figure}[t]
    \centering
    \begin{minipage}[c]{0.35\linewidth}
    \centering
    \includegraphics[width=\linewidth]{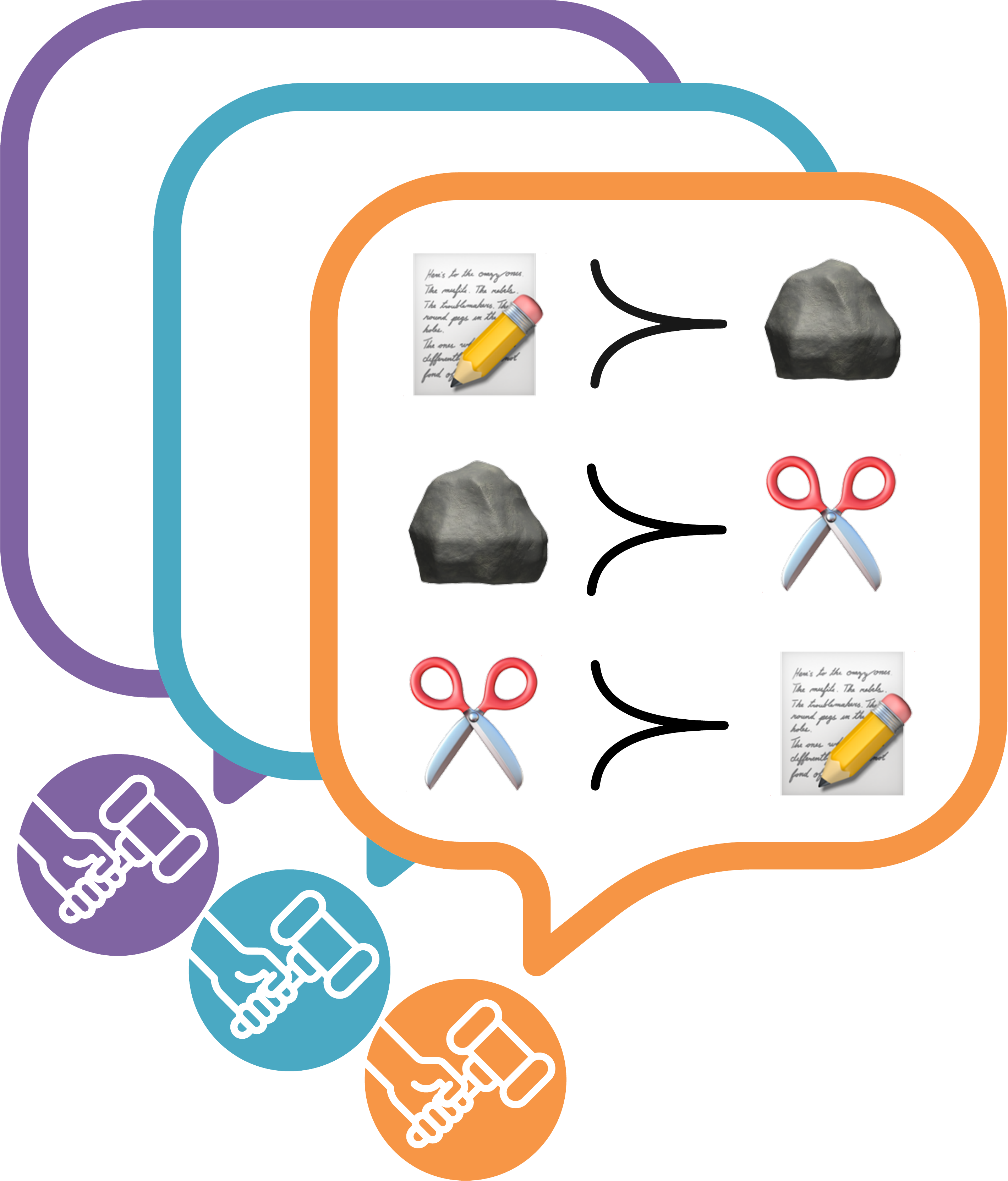}
    \end{minipage}
    \hfill
    \begin{minipage}[c]{0.61\linewidth}
    \caption{We study the problem of learning from multiple preferences over different objectives or criteria, each of which might be \textit{intransitive} (i.e., inconsistent). Such preferences are common when learning from large language model (LLM) ``judges'' \citep{zheng2023judging} that evaluate a response along multiple objectives. We propose a novel solution concept for multi-object preference fine-tuning (PFT), the \textit{Maximum Entropy Blackwell Winner (MaxEntBW)}, that remains well-defined under multiple intransitive preferences. We then derive \texttt{PROSPER}: a regression-based algorithm for computing MaxEntBWs, before using it to fine-tune LLMs from multi-objective LLM judge feedback on multiple problems.}
    \label{fig:ffig}
    \end{minipage}
\end{figure}

\noindent \textbf{1. We propose the Maximum Entropy Blackwell Winner (MaxEntBW) as a novel solution concept for multi-objective preference fine-tuning.} In contrast to classical solution concepts like the Minimax Winner that are only catered to single objectives \citep{kreweras1965aggregation, fishburn1984probabilistic, dudik2015contextual}, MaxEntBWs are well-defined under multiple objectives. MaxEntBWs also generalize the Blackwell Winners proposed by \citet{bhatia2020preference}.

\noindent \textbf{2. We derive \texttt{PROSPER}, a provably efficient, regression-based algorithm for computing MaxEntBWs.} We leverage the underlying structure of MaxEntBWs to derive a simple and scalable algorithm for computing them. Doing so requires fusing insights from \citet{ziebart2008maxent} and \citet{gao2024rebel} to first eliminate the inner-loop optimization, before reducing the outer-loop optimization down to simple regression. We prove rigorous performance and efficiency guarantees for \texttt{PROSPER}.

\noindent \textbf{3. We show that multi-objective LLM judge feedback is substantially intransitive in practice, and that \texttt{PROSPER} leads to strong downstream policies in this setting.} We first empirically verify that splitting up objectives reduces but does not eliminate intransitivity. We then apply \texttt{PROSPER} to instruction following and general chat and find that it outperforms all baselines considered.

\section{Related Work}
\raggedbottom
\noindent \textbf{Multi-Objective Preference Fine-Tuning.} 
A variety of approaches to multi-objective PFT have been explored in the literature, including minimax game formulations \citep{chakraborty2024maxmin}, distance-based optimization \citep{wang2024arithmetic,xiong2025projection}, policy aggregation \citep{zhou2024orchestrating,a2024policy,heymann2025adaptive}, distributional reward modeling \citep{siththaranjan2023distributional,halpern2025pairwise}, explicit multi-objective reward modeling \citep{rame2023rewarded,wang2024interpretable}, multi-objective reinforcement learning \citep{zhong2024panacea,wang2024conditional}, and multi-objective direct preference optimization \citep{chidambaram2025direct,zhou2024beyond,shirali2025direct,ramesh2024group}. However, all of these approaches implicitly assume that along each objective, preferences are transitive and there is a well-defined optimal policy, making it difficult to directly apply them to the more general setting we study in our work.

\noindent \textbf{Intransitive Preferences in Preference Fine-Tuning.} Drawing on ideas from social choice theory, several authors have pointed out that care is required when learning from intransitive, \textit{aggregate} preferences to avoid distorting the preferences of the underlying population of raters \citep{conitzer2024social,golz2025distortion,maura2025jackpot,liu2025statistical,xiao2025theoretical}. However, even when preference feedback is sourced from an LLM rather than a population of raters, intransitivity still frequently arises \citep{xu2025investigatingnontransitivityllmasajudge}, and can therefore stymie policy training. To deal with the fact that there is no well-defined optimal policy under intransitive preferences, a variety of authors have proposed to reframe PFT as a zero-sum game, in which one attempts to find the policy that does the best against its worst-case comparator \citep{munos2023nash, swamy2024minimaximalist, rosset2024direct, sutawika2026gainedtranslationprivilegedpairwise}. These ideas build on a rich history of similar techniques in the social choice theory \citep{kreweras1965aggregation, fishburn1984probabilistic} and bandit \citep{yue2012k, dudik2015contextual} literature. \citet{wu2025multiplayer} extend the two-player Nash-equilibrium formulation to the multi-player setting. While promising, all of these approaches assume the learner will be evaluated along a single objective. In contrast, we focus on the multi-objective setting, requiring nontrivial algorithmic innovation.

More closely to our work, several authors have explored handling multiple, potentially intransitive preferences. \citet{bhatia2020preference} pioneered the theoretical study of multi-dimensional intransitive preferences by proposing the Blackwell Winner, but stopped short of deriving a scalable algorithm capable of training modern LLMs as we do. One of the key technical innovations of the MaxEntBW is the use of entropy regularization to avoid the adversarial training \citet{bhatia2020preference}'s solution concept implicitly requires. \citet{agnihotri2025multi} take a constrained optimization perspective on multi-objective PFT, but essentially sidestep intransitivity by only optimizing the preference against a fixed reference policy, reducing the problem to standard multi-objective RL. In contrast, we directly confront the challenges of multiple intransitive preferences. Perhaps most similarly to our work, \citet{gupta2025mitigating} consider optimizing against multiple judges to deal with poor out-of-distribution judge generalization, deriving a scalable algorithm via the use of a variational relaxation. We implement a similar algorithm as one of our baselines and find that it performs worse than our proposed algorithm, \texttt{PROSPER}.

\noindent \textbf{RL from AI Feedback.} Our work falls within a growing body of research that leverages AI-generated feedback to guide reinforcement learning. In particular, we elicit AI feedback from judge models that are fed collections of rubrics \citep{cui2023ultrafeedback,gunjal2025rubrics,viswanathan2025checklists}, similar to the ``constitutional AI'' paradigm used by frontier labs \citep{bai2022constitutional}. Among rubric-based methods, \citet{cui2023ultrafeedback} evaluate responses using four global criteria rather than prompt-specific checklists, an approach shown to be less effective \citep{viswanathan2025checklists}, and \citet{gunjal2025rubrics} focus on the RLVR setting, whereas our work targets general chat and instruction following. Consequently, we primarily compare against \citet{viswanathan2025checklists} in our work. While \citet{viswanathan2025checklists} employ prompt-specific checklists, they scalarize these per-item rewards into a single score with LLM-generated weights, which can lead to intransitivity. In contrast, we dynamically reweight checklist items based on where the current model is weakest and use game-theoretic techniques to more robustly handle per-item intransitivity. We find that these differences lead to better model performance on both instruction following and general chat benchmarks.

\section{Back to Blackwell: A New Solution Concept for Multi-Objective PFT}
\subsection{Preliminaries}
We begin by discussing relevant background information.

\noindent \textbf{Notation.}
Let $\Xs$ be the space of prompts and $\Ys$ the space of responses.
We consider a class of language models (policies)
$\Pi \subseteq \{\Xs \to \Delta(\Ys)\}$, which we assume is convex and compact. For every prompt $x\in\Xs$, a policy $\pi\in\Pi$ induces a distribution over responses $\pi(x) \in \Delta(Y)$.

We assume access to a \textit{multi-objective pairwise judge},
\begin{equation}
\Ps:\ \Ys\times\Ys\times\Xs \;\to[0,1]^{m(x)},
\end{equation}
which compares responses $y,y'\in\Ys \times \Ys$ along $m(x)$ objectives, returning a vector of size $m(x)$. Rather than assuming a fixed set of criteria for every prompt (e.g., helpfulness, harmlessness), this formulation accounts for the diversity of possible prompts fed to today's LLMs: each prompt may be associated with different (and different numbers of) relevant criteria. For each $k \in [m(x)]$, $\Ps^k(y \succ y'\mid x)$ denotes the probability that $y \succ y'$, conditioned on prompt $x$ and based on criterion $k$. We emphasize that we make no assumptions about the transitivity of preferences along each objective. With a slight abuse of notation, we can extend this to a preference function between two policies $\pi, \pi' \in \Pi \times \Pi$:
\(
\Ps(\pi \succ \pi' \mid x ) \! \triangleq\! \E_{\substack{y\sim\pi\left(x\right),\ y'\sim\pi'\left( x\right)}}\!\left[\Ps\left(y \succ y' \mid x \right)\right]\!\in\!\left[0,1\right]^{m(x)}.
\)

\noindent \textbf{von Neumann Winner.} Under intransitive preferences, there is often no well-defined optimal policy (more formally, no \textit{Condorcet Winner}, \citet{brandt2016handbook}), as every response $y$ loses to some other response $y'$. Thus, we need to consider alternative solution concepts. In the single-objective setting, a popular choice is the \textit{Maximal Lottery} \citep{kreweras1965aggregation, fishburn1984probabilistic} or \textit{von Neumann Winner} \citep{dudik2015contextual}, which is the policy that is as preferred as possible against its worst-case opponent. More formally, a von Neumann Winner is a Nash Equilibrium of the zero-sum game with $\mathcal{P}$ as the payoff matrix. Due to the symmetry of the payoff matrix, simple self-play is provably efficient for computing such equilibria \citep{swamy2024minimaximalist}.

\noindent \textbf{Blackwell Winner.} A \textit{Blackwell Winner}~\cite{bhatia2020preference} is the natural extension of the von Neumann Winner to the multi-objective setting, drawing on Blackwell's notion of a \textit{target set} \citep{Blackwell1956Approachability}. For simplicity, we present this solution concept in the promptless setting. For vector $z$ and set $C$, define the $\ell_{\infty}$ \textit{distance function} as
\(
    \dist_{\infty}(z, C) = \min_{c \in C} ||z-c||_{\infty}.
\)

Suppose we want our policy $\pi$ to be preferred with at least probability $p$ to any policy $\pi' \in \Pi$ along all $m$ objectives:
\(
    \forall \pi' \in \Pi, \mathcal{P}(\pi \succ \pi') \in [p, \infty]^{m},
\)
where $C(p) \triangleq [p, \infty]^{m(x)}$ is the corresponding target set. For any $p \geq 1/2$, we can \emph{linearize} the distance function as
\begin{align*}
       \max_{\pi'\in\Pi} \text{dist}_{\infty}(\Ps(\pi \succ \pi'), C(p)) =  \max_{\pi'\in \Pi} \big ( p - \min_{w \in  \Delta_{m}} \langle w, \Ps(\pi \succ \pi')\rangle\big),
\end{align*}
since $\pi'$ can ensure that any coordinate of $\Ps^k(\pi \succ \pi') \leq 1/2$ (e.g., by setting $\pi'=\pi$). In this case, we can define the $\ell_{\infty}$ Blackwell Winner \citep{bhatia2020preference} as
\begin{equation}\label{bw-no-prompt}
\argmin_{\pi \in \Pi} \max_{\pi' \in \Pi} \big ( p - \min_{w \in  \Delta_{m}} \langle w, \Ps(\pi \succ \pi')\rangle\big).
\end{equation}

\subsection{The Maximum Entropy Blackwell Winner}

Note that removing the constant shift $p$ in \cref{bw-no-prompt} does not affect the solution. Then, reintroducing the prompts $x$, we can express the $\ell_{\infty}$ Blackwell Winner as the solution to
\begin{equation}
    \max_{\pi\in\Pi}\ \min_{w:\mathcal{X}\to\Delta_{m(x)}}\ \min_{\pi'\in\Pi}\quad
 \mathbb{E}_{x}\!\left[\left\langle w(x),\, \Ps(\pi\succ\pi'\mid x)\right\rangle\right].
\end{equation}
While perhaps reasonable for small-scale problems, performing a \textit{best response} over adversary policies $\pi'$ as in the above is infeasible at the scale of today's foundation models. Furthermore, adversarially selecting both the criteria and comparator policy could lead to situations where \textit{all} policies look bad, even though some are qualitatively better than others. Thus, to control the strength of the adversary, we propose to KL regularize it to a reference policy $\pi_{\mathsf{ref}}$:
\begin{align}
\max_{\pi\in\Pi}\ \min_{w:\mathcal{X}\to \Delta_{m(x)}}\ \min_{\pi'\in\Pi}\ 
\mathbb{E}_{x}\!\left[
\langle w(x),\, \Ps(\pi\succ \pi' \mid x)\right\rangle
+\beta\,\KL\!\left(\pi'\left( x\right)\,\|\,\pi_{\reff}\left(x\right)\right)
], \label{eq:opt_w_kl}
\end{align}
where $\beta>0$ controls the strength of the regularization.
Beyond being perhaps the natural choice, the use of KL regularization on the adversary will allow us to elide adversarial training, as we discuss in greater detail below. For now though, let us define the game value of a policy $\pi$ as
\begin{equation}
\begin{split}
V(\pi)\ \triangleq\ \min_{w:\Xs\to \Delta_{m(x)}}\ \min_{\pi'\in\Pi}\ 
\E_{x}\big[\left\langle w(x), \Ps(\pi \succ \pi' \mid x)\right\rangle+\ \beta\,\KL\!\left(\pi'(x)\,\|\,\pi_{\reff}(x)\right)\big].    
\end{split}
\end{equation}
We are now equipped to define the \textit{Maximum Entropy Blackwell Winner (MaxEntBW)}:
\begin{definition}[Maximum Entropy Blackwell Winner (MaxEntBW)]
A policy $\pi^\star\in\arg\max_{\pi\in\Pi} V(\pi)$ is called a \emph{Maximum Entropy Blackwell Winner (MaxEntBW)}.
\label{def:maxentbw}
\end{definition}
In words, a MaxEntBW is a policy that compares favorably to policies in its local neighborhood, regardless of the objective it is evaluated under -- a nuanced form of robustness. Observe that this solution concept is well-defined even under multiple, globally inconsistent sets of preferences.

In contrast to the single-objective setting, the Minimax Theorem \citep{vonneumann1928} doesn't hold in the multi-objective setting when we consider the worst case over objectives \citep{Blackwell1956Approachability}. Thus, the \textit{order of play} matters: the value of the game is different depending on whether learner $\pi$ or adversary $\pi'$ moves first \citep{roth_games_in_learning}. This means that the self-play approaches \citep{swamy2024minimaximalist} that were sufficient in the single-objective setting cannot be directly extended to the multi-objective setting, seemingly indicating that adversarial training may be necessary.

\section{\texttt{PROSPER}: PReference Optimization with a Single Player over Entire Rubrics}

\begin{algorithm*}[t]
\caption{\PROSPER: PReference Optimization with a Single Player over Entire Rubrics}\label{alg: multi-preference}
\begin{algorithmic}[1]
\State \textbf{Input:} A dataset of prompts $\mathcal{D}$, reference policy $\pi_{\reff}$, judge $\Ps:\Ys \times \Ys \times \Xs \to [0,1]^{m(x)}$, KL penalty coefficient $\beta$, time steps $T$, step size $\eta$, policy class $\Pi=\{\pi_{\theta}:\theta\in\Theta\}$, sample size $M$ for gradient estimation.
\State \textbf{Initialize:} $\pi_{\theta_0} = \pi_{\reff}$.
\For{$t = 0, \dots, T-1$}
    \State For each $x\in\mathcal{D}$, collect $\mathcal{D}(x)=\{(y_1,...,y_M, y_1',...,y_M', z,z')\}$, where $y_i\sim\pi_{\theta_t}\left(x\right)$, $y'_j\sim\pi_\reff\left( x\right)$, $z\sim\pi_{\theta_t}\left(x\right)$, $z'\sim \pi_{\reff}\left(x\right)$.
    \State For each $x\in\mathcal{D}$, compute the worst-case objective $\widehat{k}(x)$ via Eq. \ref{eq:hat-k}.
    
    \State  For each $x\in\mathcal{D}$, compute gradient estimates $\widehat{g}_t(x,z)$ and $\widehat{g}_t(x,z')$, where
    \begin{equation*}
    \widehat{g_t}(x,z)=
    \frac{\sum_{j=1}^M\Ps^{\widehat{k}(x)}(z \succ y_j' \mid x)\exp\InParentheses{-\frac{1}{M\beta}\sum_{i=1}^M\Ps^{\widehat{k}(x)}(y_i \succ y_j' \mid x)}}{\sum_{j=1}^M\exp\InParentheses{-\frac{1}{M\beta}\sum_{i=1}^M\Ps^{\widehat{k}(x)}(y_i \succ y_j' \mid x)}}.
    \end{equation*}
    \State Update policy parameters via solving a square loss regression problem:
\[
\theta_{t+1}=\argmin_{\theta\in\Theta}\sum_{\substack{x \in \mathcal{D} \\ z, z' \in \mathcal{D}(x)}}\InParentheses{\frac{1}{\eta}\InParentheses{\ln \frac{\pi_\theta(z \mid x)}{\pi_{\theta_t}(z \mid x)}- \ln \frac{\pi_\theta(z' \mid x)}{\pi_{\theta_t}(z' \mid x)}}-\InParentheses{\widehat{g}_t(x,z) - \widehat{g}_t(x,z')}}^2.
\]

\EndFor
\end{algorithmic}
\end{algorithm*}

Despite the apparent complexity of solving \cref{eq:opt_w_kl} to compute a MaxEntBW (\cref{def:maxentbw}), the underlying structure of the optimization problem admits a strikingly simple solution, as we now explore. We proceed by working from the inner-most optimization problem outwards.

\noindent \textbf{Key Step \#1: Entropy Regularization Ensures a Closed-Solution for $\pi'$.} Fix a $\pi$ and $w$. Then, we need to solve
\begin{align}
\min_{\pi'\in\Pi}\ 
\mathbb{E}_{x}\!\left[
\langle w(x),\, \Ps(\pi\succ\pi'\mid x)\right\rangle
+\beta\,\KL\!\left(\pi'\left( x\right)\,\|\,\pi_{\reff}\left( x\right)\right)
].
\end{align}
 Using $\mathcal{P}_\pi(y'\mid x)=\E_{y\sim\pi(x)}[\Ps(y\succ y'\mid x)]$, we know from \citet{ziebart2008maxent} that this problem has solution
\begin{equation}
    \pi'_\star(y'\mid x)=\frac{\pi_{\reff}(y'\mid x)\exp(-\InAngles{w(x),\mathcal{P}_\pi(y'\mid x)}/\beta)}{Z(w,\pi\mid x)},
\end{equation}
where
\begin{equation}
    Z(w,\pi\mid x)=\E_{y' \sim \pi_{\reff}(x)} [\exp(-\InAngles{w(x),\mathcal{P}_\pi(y'\mid x)}/\beta)]
\end{equation}
is the \textit{partition function} that ensures the above sums to 1 across $\mathcal{Y}$. As in \citet{gupta2025mitigating}, we can then plug this closed-form expression back into \cref{eq:opt_w_kl} to arrive at
\begin{equation}
    \max_{\pi \in \Pi} \min_{w:\mathcal{X}\to \Delta_{m(x)}} \mathbb{E}_{x}[-\beta\log Z(w,\pi|x)]. \label{eq:logz}
\end{equation}
Observe that this optimization problem does not require adversarial training between policies to solve.

\noindent \textbf{Key Step \#2: Simplifying $w$ down to a Prompt-Wise Minimum.} Next, we observe that \cref{eq:logz} is \textit{concave} in $w(x)$. Thus, for $\forall x \in \mathcal{X}$, we know that $w^{\star}(x)$ must be a vertex of the simplex $\Delta_{m(x)}$ (i.e., a standard basis vector). As the set of objectives is usually relatively small in practice, we can simply take a minimum prompt-wise across different objectives to eliminate the optimization over $w$. More formally, we can define the $k$th partition function
\begin{equation}
    Z^k(\pi\mid x) = \E_{y' \sim \pi_{\reff}(x)} \!\left[\exp\!\Big(-\mathcal{P}_\pi^k(y'\mid x)/\beta\Big)\right], 
\end{equation}
and the worst-case coordinate as
\begin{equation}
    k^\star(x) \;=\; \argmin_{k \in [m(x)]} -\beta\log Z^k(\pi\mid x).
\end{equation}
Then, we can further simplify down \cref{eq:logz} to:
\begin{equation}\label{eq:zstar}
        \max_{\pi \in \Pi} \mathbb{E}_{x}\InBrackets{-\beta\log Z^{k^{\star}(x)}(\pi|x)}.
\end{equation}
We have now reduced our optimization problem, which initially had three players, to a single player.

\noindent \textbf{Key Step \#3: Using Regression to Approximate Online Mirror Descent.} We now observe that \cref{eq:zstar} is concave in $\pi$. This is because each $- \beta\log Z^k(\pi\mid x)$ is concave in $\pi$ (see Lemma~\ref{lem: concave_Z} in Appendix~\ref{sec:detailed_proof}) and the minimum of concave functions is also concave. We can therefore reduce this problem to online convex optimization \citep{hazan2023introductiononlineconvexoptimization} and use Online Mirror Descent (in particular, with KL as the Bregman Divergence) for policy optimization.

The first step of a descent algorithm is calculating the gradient. Algebra tells us that the (functional) gradient with respect to $\pi$ at some response $z \in \mathcal{Y}$ of \cref{eq:zstar} is
\begin{align}
 g(z\mid\pi, x) =  \frac{\E_{y'\sim\pi_{\reff}(x)}[\mathcal{P}^{k^{\star}(x)}(z \succ y' \mid x)\exp(-\mathcal{P}_\pi^{k^{\star}(x)}(y'\mid x)/\beta)]}{Z^{k^{\star}(x)}(\pi \vert x)}.      \label{eq:gradz}
\end{align}
On small-scale problems, we could follow this gradient by performing mirror descent independently at each prompt $x \in \mathcal{X}$ and response $y \in \mathcal{Y}$. Unfortunately, both $\mathcal{X}$ and $\mathcal{Y}$ are too large in modern foundation model training for doing so to be practically feasible. In response, we use the regression-based reformulation of mirror descent proposed by \citet{gao2024rebel, gao2025regressingrelativefutureefficient, brantley2025accelerating,zhou2024orchestrating} to arrive at the following update:
\begin{align*}
    \theta_{t+1} =\argmin_{\theta \in \Theta} \mathbb{E}_{\substack{x, z\sim\pi_{\theta_t}(x), \, z'\sim\pi_{\reff}(x)}}
\Big( \frac{1}{\eta} 
    \Big( 
        \ln \frac{\pi_\theta(z \mid x)}{\pi_{\theta_t}(z \mid x)} - 
        \ln \frac{\pi_\theta(z' \mid x)}{\pi_{\theta_t}(z' \mid x)} \Big) - \big( g(z\mid\pi_{\theta_t}, x) - g(z'\mid\pi_{\theta_t}, x) \big) 
\Big)^2.
\end{align*}
As proved by \citet{gao2024rebel}, solving the above optimization problem would correspond to exact mirror descent. While we are unlikely to be able to do this in practice, we can still learn a policy with strong performance guarantees under approximation error, as we show in Section~\ref{sec:theoretical}.

\noindent \textbf{Putting It All Together in the Finite-Sample Setting.} In practice, we also need to estimate the infinite-sample gradient (Eq. \ref{eq:gradz}) based on samples from the current policy and the reference policy. Thus, we draw $M$ samples $y_{1:M} \sim \pi(x)$ and another $M$ samples
$y'_{1:M} \sim \pi_{\reff}(x)$ independently. We then approximate the partition function
\begin{equation}\label{eq:hat-h}
\widehat{Z}^k(x):=\frac{1}{M}\sum_{j=1}^M\InBrackets{\exp{\InParentheses{-\frac{1}{M\beta}\sum_{i=1}^M\Ps^k(y_i \succ y'_j \mid x)}}}, 
\end{equation}
\begin{equation}\label{eq:hat-k}
\widehat{k}(x)=\argmin_{k\in[m(x)]}-\beta\log\widehat{Z}^k(x),
\end{equation}
and arrive at an estimate of the gradient $\widehat{g_t}(x,z)$:
\begin{equation}
    \frac{\sum_{j=1}^M\Ps^{\widehat{k}(x)}(z \succ y_j' \mid x)\exp(-\frac{1}{M\beta}\sum_{i=1}^M\Ps^{\widehat{k}(x)}(y_i \succ y_j' \mid x))}{\sum_{j=1}^M\exp(-\frac{1}{M\beta}\sum_{i=1}^M\Ps^{\widehat{k}(x)}(y_i \succ y_j' \mid x))}.
\end{equation}
We then regress this gradient estimate in terms of the policy to perform updates as in \citet{gao2024rebel}. We call this approach \texttt{PROSPER}: PReference Optimization with a Single Player over Entire Rubrics and outline it in Algorithm \ref{alg: multi-preference}.

\subsection{Performance Guarantee for \texttt{PROSPER}}\label{sec:theoretical}
Recall that the game value of a policy $\pi$ is denoted by $V(\pi)$. Let $V(\pi^{\star})$ be the optimal value. We now prove that under the assumption that the regression problem can be solved accurately in-distribution, \PROSPER~can efficiently learn a policy $\hat{\pi}$ such that $V(\hat{\pi})$ is close to $V(\pi^\star)$. Intuitively, this means $\hat{\pi}$ competes around as favorably as possible with policies in its local neighborhood, regardless of the objective it is judged under. More formally, we assume that
\begin{assumption}\label{ass: regression}
For all $t \in [T]$, we have that for some $\epsilon$,
\begin{align*}
      \mathbb{E}_{\substack{x, \\ z\sim \pi_{\theta_t}(x), \\ z'\sim\pi_{\reff}(x)}} \Bigg( 
\frac{1}{\eta} \Bigg( 
\ln \frac{\pi_{\theta_{t+1}}(z|x)}{\pi_{\theta_t}(z|x)} 
- \ln \frac{\pi_{\theta_{t+1}}(z'|x)}{\pi_{\theta_t}(z'|x)}\Bigg) 
- \left( g(z\mid\pi_{\theta_{t}}, x) - g(z'\mid\pi_{\theta_{t}}, x) \right) 
\Bigg)^2 \leq \epsilon.
\end{align*}
\end{assumption}
Intuitively, this assumption says that there is a policy in our policy class that is able to accurately fit the difference of gradients. We believe this is a reasonable assumption for modern deep networks. Next, given some policy $\pi$, we denote the \textit{concentrability coefficient} as $C_{\pi_{\reff}\rightarrow\pi}=\max_{x,y}\frac{\pi(y\mid x)}{\pi_{\reff}(y\mid x)}$ \citep{kakade2002approximately}, which quantifies how well the reference policy $\pi_{\reff}$ covers $\pi$. We are now ready to state our main theoretical result:

\begin{restatable}{theorem}{FormalBound}\label{thm: formal_bound}
   Under \cref{ass: regression}, after $T$ iterations, there exists some $\hat{\pi} \in \{\pi_{\theta_1},\cdots,\pi_{\theta_T}\}$ such that $V(\pi^\star)-V(\hat{\pi})\le O\InParentheses{\sqrt{1/T}+\sqrt{C_{\pi_{\reff}\rightarrow\pi^\star}\epsilon}}$, where 
$\pi^\star\in\argmax_{\pi\in\Pi} V(\pi)$ denotes an arbitrary MaxEntBW.
\end{restatable}

In words, this theorem says that as long as supervised learning works (i.e., $\epsilon$ is small), then \texttt{PROSPER} can efficiently compute an approximate MaxEntBW. We defer the detailed proof to Appendix \ref{sec:detailed_proof} and turn our attention to empirically validating our approach on LLM fine-tuning problems.

\section{LLM Fine-Tuning Experiments Setup}
We now outline the setup used for our LLM experiments. We defer all hyperparameters to App. \ref{app:hyperparameters}.

\noindent \textbf{Training Dataset.}
We train on the \textsc{WildChecklists} dataset~\citep{viswanathan2025checklists}. \textsc{WildChecklists} is based on \textsc{WildChat}, a collection of conversations between users and AI language models \citep{zhao2024wildchat}. As described in \citet{viswanathan2025checklists}, \textsc{WildChecklists} was created by using Qwen2.5-72B-Instruct to generate \emph{prompt-specific} rubrics to guide LLM judge evaluation of candidate responses. Note that our framework -- which does not assume a fixed set of criteria across all prompts -- is able to flexibly accommodate this more nuanced feedback. One can view each checklist item as one of the multiple objectives discussed in our preceding formulation.

\noindent \textbf{Policy and Judge Models.}
We fine-tune the 3B and 7B parameter Qwen2.5-Instruct models \citep{qwen2025qwen25technicalreport} and use Qwen3-14B as our judge model \citep{yang2025qwen3technicalreport}. We selected Qwen3-14B due to the relatively high agreement it demonstrated with author judgments across multiple criteria. We include the full prompt templates in Appendix~\ref{sec:Implementation_details} and use the standard averaging over both input orderings to mitigate the positional bias of pairwise judges \citep{zheng2023judging}. All LLM judges score response pairs along a checklist item using a 5-point Likert scale. We query the judge 5 times and average to reduce variance before normalizing to $[0, 1]$.

\noindent \textbf{Benchmarks.} We primarily evaluate our trained models on  \textsc{AlpacaEval}~\citep{dubois2024alpacaeval} and \textsc{Arena-Hard}~\citep{li2024arenahard}, as these benchmarks are in-domain for the \textsc{WildChecklists} training dataset. Additionally, to ensure our training procedure doesn't cause regression on other capabilities, we also report results on multiple-choice QA and reasoning benchmarks: \textsc{IFEval}~\citep{zhou2023ifeval}, \textsc{MMLU}~\citep{hendrycks2020mmlu}, \textsc{ARC}~\citep{clark2018arc}, \textsc{HellaSwag}~\citep{zellers2019hellaswag}, and \textsc{TruthfulQA}~\citep{lin2021truthfulqa}.

\noindent \textbf{Baselines.} Given the similarity of the training data, we primarily evaluate our methods against the \textsc{RLCF} (RL from Checklist Feedback) method of \citet{viswanathan2025checklists}. For each prompt, candidate response, and checklist item, \textsc{RLCF} generates a score by querying the LLM judge. Another LLM (Qwen2.5-72B-Instruct in our experiments) then generates a set of weights used to combine the per-item scores into a per-response reward. Thus, the primary distinction between our method and RLCF is the use of a more principled, game-theoretic aggregation of per-item scores. Less important distinctions include the use of DPO \citep{rafailov2024directpreferenceoptimizationlanguage} on the highest and lowest scoring responses for a prompt rather than gap size-sensitive REBEL \citep{gao2024rebel} for policy optimization and our use of \textit{pairwise} judges rather than a judge that evaluates responses in isolation. To ensure a fair comparison, we re-implement \textsc{RLCF} within our optimization pipeline using REBEL as the policy optimizer and report the resulting performance in the main text. We further evaluate the authors’ released checkpoints and find their performance to be comparable to our re-implementation (see \cref{subsec: 3b}). Both, however, are consistently outperformed by our method.

\noindent \textbf{Ablations.} In addition to evaluating our full method, \texttt{PROSPER}, we evaluate two simplifications, each of which correspond to dropping one of the inner loop optimizations.

\noindent \underline{Ablation \#1: Single-Objective Optimization.} We ablate the importance of explicitly optimizing over multiple objectives by instead asking the LLM judge for a single, aggregate score (i.e., $\forall x \in \mathcal{X}, m(x) = 1$ that reflects all checklist items \textit{jointly}). We refer to this variant as \fullcheck. 

\noindent \underline{Ablation \#2: Fixed Competitor Policy.} Rather than adversarially selecting a competitor policy, we can instead compare our policy to $\pi_{\reff}$. This corresponds to taking $\beta \to \infty$ or a \textit{variational bound} on our full objective \citep{gupta2025mitigating}. Thus, we refer to this variant as \nogame. 

 Across all variants of our method, we use $M=2$ samples for gradient estimation and sample $4$ responses per prompt, rather than a single pair $(z, z')$ as in Algorithm \ref{alg: multi-preference}. We also pool pairs across prompts and only train on the response pairs with the largest gaps in gradient vectors, which we find improves the quality of the final learned model. Intuitively, one can view this process as filtering out response pairs where the LLM judge was not confident in its choice.

\section{\texttt{PROSPER} Enables Efficient Multi-Objective PFT From LLM Judge Feedback}

Our experiments focus on answering four key questions:

\noindent \textbf{1. How much of an issue is LLM judge intransitivity in practice?} We compute the fraction of prompts for which LLM judges have intransitive preferences and for which Condorcet Winner exists. We also evaluate how this fraction changes if we decompose a checklist into individual items and ask the judge to evaluate each independently.

\noindent \textbf{2. Is \texttt{PROSPER} able to optimize multi-objective judge preferences?} We compute the pairwise win-rate between all trained pairs of trained policies as in \citet{munos2023nash} to quantify how well the policies produced by \texttt{PROSPER} optimize nuanced, multi-criteria LLM judge preferences.

\noindent \textbf{3. Do the policies trained via \texttt{PROSPER} exhibit qualitatively strong behavior?} Beyond optimizing the MaxEntBW objective, we also evaluate how well the policies produced by \texttt{PROSPER} perform on the \textsc{AlpacaEval} (i.e., instruction following), \textsc{Arena-Hard} (i.e., general chat), \textsc{IFEval}, \textsc{MMLU}, \textsc{ARC}, \textsc{HellaSwag}, and \textsc{TruthfulQA} benchmarks. We primarily evaluate on the first 2 benchmarks as they are in-domain for our training set \textsc{WildChecklists}, with the rest serving as regression tests.

\noindent \textbf{4. How important are the inner loop optimizations for \texttt{PROSPER}'s practical performance?} We ablate the effect of the adversarially chosen competitor (\texttt{PROSPER-VB}) and adversarially weighted objectives (\texttt{PROSPER-JC}) on trained policy performance across all seven benchmarks.

We now answer these questions in order.

\subsection{Multiple Objectives Mitigate but Don't Fully Eliminate LLM Judge Intransitivity}\label{sec:intransitivity}

\begin{figure}
    \centering
        \includegraphics[width=.35\columnwidth]{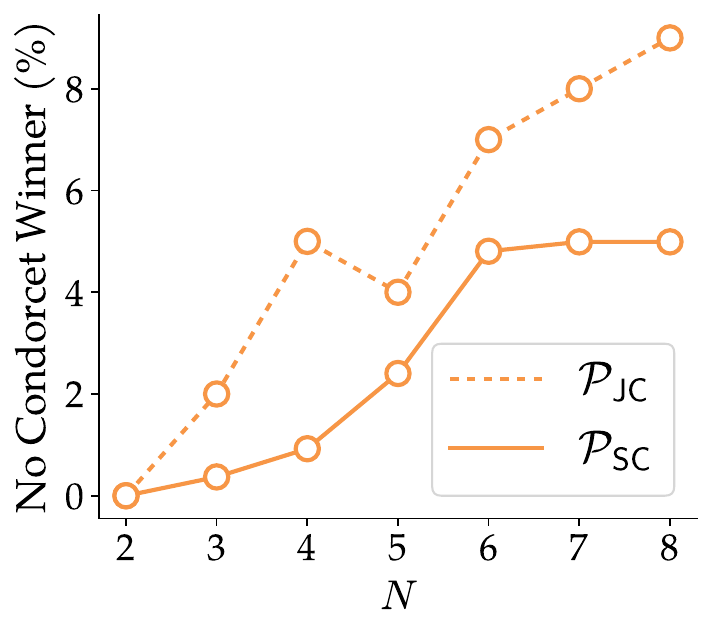}
        \includegraphics[width=.35\columnwidth]{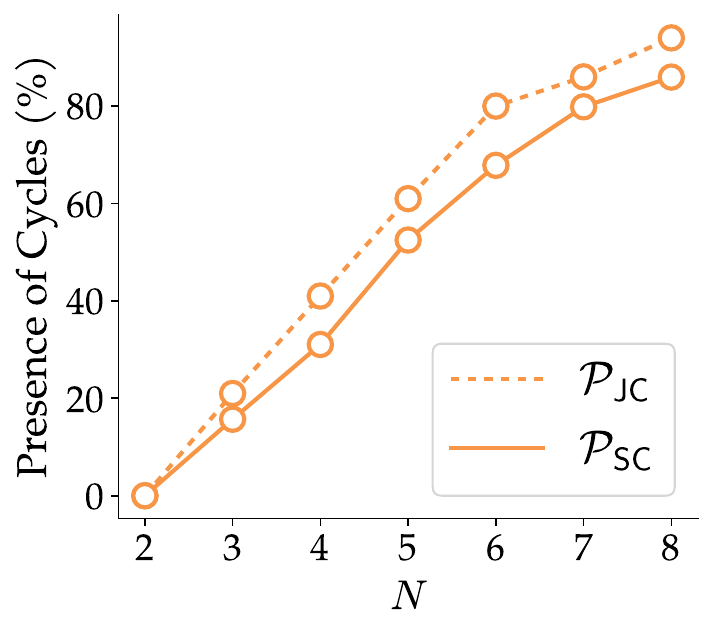}
    \caption{On a held-out set of prompts from \textsc{WildChecklists}, we report the fraction of prompts with no Condorcet Winner \textit{(left)} and Intransitive Preferences \textit{(right)} for both the $\Ps_{\mathsf{JC}}$ (joint check) and $\Ps_{\mathsf{SC}}$ (single check) judges, where $N$ denotes the number of generated responses. We see that splitting up a rubric into multiple items before passing it to the LLM judge (i.e., using $\Ps_{\mathsf{SC}}$ rather than $\Ps_{\mathsf{JC}}$) reduces but doesn't eliminate inconsistent preferences.}
\label{fig: intrans}
\end{figure}

On a held-out set of 100 prompts from the \textsc{WildChecklists} dataset, we evaluate the inconsistency of our LLM judge on $N$ responses from the Qwen2.5-7B-Instruct base model. We compute both the percentage of prompts for which we see intransitive preferences (i.e., those where the pairwise preferences form a cycle of any length \citep{xu2025investigatingnontransitivityllmasajudge}) and there is no Condorcet Winner (i.e., those where no fixed response is preferred to all others). We note that these two quantities are not the same, as cycles may only occur across relatively weak responses. We compare two judges: $\Ps_{\mathsf{JC}}$ (joint check) and $\Ps_{\mathsf{SC}}$ (single check). For $\Ps_{\mathsf{JC}}$, the LLM judge takes the full checklist as input and outputs a single, aggregate score to indicate how much it prefers one response over the other. For $\Ps_{\mathsf{SC}}$, the LLM judge evaluates both responses along each checklist item independently. We compute metrics for each (prompt, item) pair before reporting the average over items and prompts.

In Figure \ref{fig: intrans}, we see that both judges frequently have intransitive preferences for large enough batch size $N$, echoing the findings of \citet{xu2025investigatingnontransitivityllmasajudge}. This underscores the importance of robust algorithms like \texttt{PROSPER} for learning from such inconsistent feedback. We also see that splitting up a checklist into multiple items (i.e., using $\Ps_{\mathsf{SC}}$ rather than $\Ps_{\mathsf{JC}}$) improves both metrics but does not completely eliminate the problems of having no Condorcet Winner and intransitive preferences. Thus, while checklist / rubric rewards are an important first step, we still need to take care in the multi-objective setting to handle LLM judge intransitivity.

\subsection{\texttt{PROSPER} Effectively Optimizes Multi-Objective LLM Judge Preferences}
\begin{figure}[H]
\centering
   \begin{minipage}[c]{.35\columnwidth}
   \centering
    \includegraphics[width=\columnwidth]{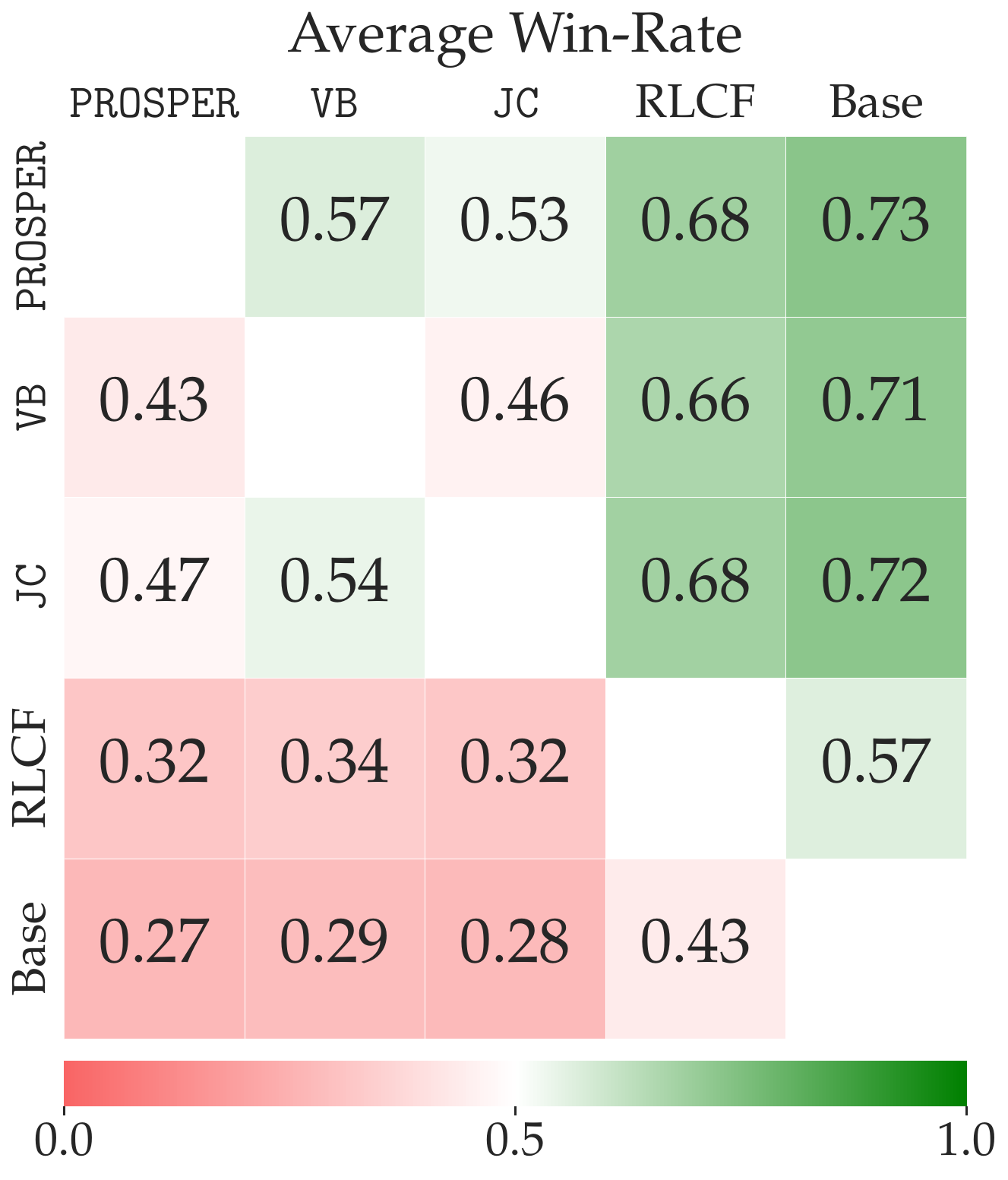}   
   \end{minipage}
   \hfill
   \begin{minipage}[c]{.61\columnwidth}
    \caption{We consistently see policies trained via \texttt{PROSPER} outperform RLCF (roughly $2/3$ of the time), baseline (roughly $3/4$ of the time), and ablation method policies, as measured by LLM judge win-rates on held-out prompts. This indicates that \texttt{PROSPER} is able to more effectively optimize nuanced, multi-criteria LLM judge preferences than the other methods we consider.
    \label{fig:wr}} 
   \end{minipage}
\end{figure}
Next, we ask the LLM judge to compare all trained 7B policies (\texttt{PROSPER}, \texttt{PROSPER-JC}, \texttt{PROSPER-VB}, RLCF, and the base Qwen2.5-7B-Instruct) on a set of 500 held-out prompts from the test set. For each pair of policies and prompt, we sample responses from each policy. We then ask our judge to evaluate each pair of responses along all criteria and pick a winning policy for each. We then average over checklist items, before finally averaging over prompts to compute a \textit{win rate}.

In Figure \ref{fig:wr}, we see that \texttt{PROSPER} produces policies that better optimize for judge preferences than the base policy and RLCF baseline. Zooming in on the top row, we see that \texttt{PROSPER} produces policies that generate responses preferred to the base policy responses roughly $3/4$ of the time and the RLCF policy responses roughly $2/3$ of the time. Furthermore, we see that the full version of \texttt{PROSPER} outperforms both the \texttt{PROSPER-JC} and \texttt{PROSPER-VB} ablations, indicating that both inner loop optimizations improve efficacy at optimizing rubric-based LLM judge preferences.

\subsection{\texttt{PROSPER} Produces Policies With Strong Performance on Multiple Benchmarks}
\begin{table*}[t]
\centering
\renewcommand{\arraystretch}{1.2}
\begin{tabular}{lccccc}
\hline
\textbf{Model} 
& \multicolumn{2}{c}{\textbf{Arena-Hard}} & \multicolumn{2}{c}{\textbf{AlpacaEval}} & \\
\hline
& \textbf{Vanilla} & \textbf{Style-Controlled}
& \textbf{Vanilla} & \textbf{Length-Controlled} &\\
\hline

\textit{Qwen2.5-7B-Instruct} & 42.4 & 44.2 & 37.1& 25.32&\\
\quad + RLCF      & 42.5 & 43.9 &  41.4&17.24&\\
\quad + \fullcheck  &  44.2 & 42.0 &  \underline{55.3} & \textbf{38.21}&\\
\quad + \nogame     & \underline{47.6} & \underline{45.5} &  51.2&33.64 &\\
\quad + \PROSPER  & \textbf{49.2} & \textbf{46.1} &  \textbf{55.4} & \underline{37.61} &\\
\hline
\end{tabular}
\caption{Highest score is \textbf{bolded} and second highest \underline{underlined}. Across instruction following and general chat benchmarks, \texttt{PROSPER} consistently outperforms the RLCF baseline and \texttt{PROSPER-JC} and \texttt{PROSPER-VB} baselines at the 7B parameter scale.}
\label{tab: arena+alpaca}
\end{table*}
\begin{table*}[!t]
\centering
\renewcommand{\arraystretch}{1.25}
\begin{tabular}{lcccccccc}
\hline
\textbf{Model} 
& \textbf{IFEval} 
& \textbf{MMLU} 
& \textbf{ARC} 
& \textbf{HellaSwag} 
& \textbf{TruthfulQA} 
& \textbf{Avg} 
& \\
\hline

\textit{Qwen2.5-7B-Instruct} & \underline{71.16} & 63.39& 85.23& \textbf{74.54}& \textbf{68.42}& \textbf{72.55} &  & \\
\quad + RLCF                   &67.83 &\textbf{63.47} &85.07 & \underline{74.18}& \underline{65.97}& 71.30& & \\
\quad + \fullcheck         & 69.87&63.37 &\underline{85.41} & \underline{74.18}& 65.61&71.69 & & \\
\quad + \nogame      & \textbf{73.75}&63.36 &85.32 &74.04 & 65.6&\underline{72.32} & & \\
\quad + \PROSPER                & 69.50&\underline{63.42} &\textbf{85.49} & 74.12& \underline{65.97} &71.70 & &\\
\hline
\end{tabular}
\caption{Highest score is \textbf{bolded} and second highest \underline{underlined}. By comparing the first and last row of the table, we see that \texttt{PROSPER} does not lead to significant model commonsense reasoning and general knowledge QA capability degradation at the 7B parameter scale.}
\label{tab: others}
\end{table*}

In Table \ref{tab: arena+alpaca}, we report the performance of the 7B policies produced by all methods on \textsc{AlpacaEval 2.0} (instruction following) and \textsc{Arena-Hard} (general chat). For both benchmarks, we use gpt-5-mini as the LLM judge. We see that \texttt{PROSPER} outperforms RLCF and the \texttt{PROSPER-JC} and \texttt{PROSPER-VB} ablations on \textsc{Arena-Hard} and Vanilla \textsc{AlpacaEval} and is a close second place for length-controlled \textsc{AlpacaEval}. This indicates that beyond merely optimizing judge preferences, \texttt{PROSPER} trains qualitatively strong policies. This also provides evidence that beyond being robust in theory, the MaxEntBW solution concept appears to correspond to qualitatively strong policies. Furthermore, by comparing the first and last rows of Table \ref{tab: others}, we see that \texttt{PROSPER} does not degrade model commonsense reasoning and knowledge QA capabilities (which are out-of-domain for our training set), indicating that neither our solution concept nor training scheme is so adversarial as to produce overly conservative policies. We find similar results at the 3B scale -- see Appendix \ref{subsec: 3b}.

\section{Conclusion}
We introduce the \textit{Maximum Entropy Blackwell Winner}: a solution concept for learning from multiple, potentially intransitive preferences. This is frequently the case when learning from rubric-based LLM judges. We then derive \texttt{PROSPER}, a provably efficient algorithm for computing MaxEntBWs, before using it to fine-tune LLMs at the 3B and 7B parameter scales. %

\section{Acknowledgments}
GKS, JZ, and ZSW were supported in part by a STTR grant. We thank Zhaolin Gao for in-depth feedback on how best to implement \texttt{PROSPER}. GKS also thanks Anca Dragan for a conversation at ICML 2023 about the two root causes of intransitive preferences that eventually inspired this project.

\section{Contribution Statements}
\begin{itemize}
    \item \textbf{JZ} and \textbf{LZ} co-led the project, deriving the core algorithm and implementing it on language models. Together, they wrote the first draft of the paper.
    \item \textbf{KG} helped with running experiments and writing the paper. \textbf{ZY} helped with evaluating trained models across multiple benchmarks.
    \item \textbf{GS} proposed the project, suggested the algorithmic approach, helped debug experiments, helped advise the project, and wrote most of the final draft.
    \item \textbf{ZSW} advised the project, helping with formalizing the algorithm, writing the paper, and running experiments.
\end{itemize}

\bibliography{reference}
\bibliographystyle{abbrvnat}

\newpage
\appendix
\crefalias{section}{appendix}
\onecolumn

\section{Supplementary Results}\label{subsec: 3b}
We provide the benchmark performance on Qwen2.5-3B-Instruct in Table~\ref{tab:sup_1} and Table~\ref{tab:sup_2}.
\begin{table}[h]
\centering
\begin{tabular}{lccccc}
\hline
\textbf{Model} 
& \multicolumn{2}{c}{\textbf{Arena-Hard}} & \multicolumn{2}{c}{\textbf{AlpacaEval}} & \\
\hline
& \textbf{Vanilla} & \textbf{Style-Controlled}
& \textbf{Vanilla} & \textbf{Length-Controlled} &\\
\hline

\textit{Qwen2.5-3B-Instruct} & 20.6 & 19.6 & 24.1&12.79&\\
\quad + RLCF               & 22.8 &  22.9 &26.2& 11.62&  \\
\quad + \fullcheck        & 22.9 & 23.4 &  24.7& 7.4&\\
\quad + \nogame      & \underline{25.1} & \underline{23.6} & \underline{34.8} & \underline{19.99} &\\
\quad + \PROSPER  & \textbf{26.0} & \textbf{24.6} & \textbf{35.8} & \textbf{20.94} &\\
\hline
\end{tabular}

\caption{\textbf{Model performances on instruction-following and general-chat benchmarks (fine tuning Qwen2.5-3B-Instruct).}  Highest score presented in \textbf{bold} and second highest \underline{underlined}. Across instruction-following and preference-alignment benchmarks (AlpacaEval 2.0 and Arena-Hard) evaluated using gpt-5-mini as the judge, \texttt{PROSPER} consistently outperforms the \texttt{PROSPER} ablations (joint-check and variational-bound), while also improving over the RLCF baseline.}
\label{tab:sup_1}
\end{table}

\begin{table}[h]
\centering
\begin{tabular}{lcccccccc}
\hline
\textbf{Model} 
& \textbf{IFEval} 
& \textbf{MMLU} 
& \textbf{ARC} 
& \textbf{HellaSwag} 
& \textbf{TruthfulQA} 
& \textbf{Avg} 
& \\
\hline
\textit{Qwen2.5-3B-Instruct} &\textbf{59.70} & \textbf{56.45}&\underline{79.43} &\textbf{67.68} &\textbf{64.62} & \textbf{65.58}& & \\
\quad + RLCF & 58.04& 56.28& 79.35&66.88 & 62.17& 64.54& & \\
\quad + \fullcheck             &53.24 &56.16 & 79.27& 66.87& 61.69& 63.45& & \\
\quad + \nogame               &57.30 & \underline{56.34}&\textbf{79.52} &\underline{67.36} &\underline{62.67} & 64.64& \\
\quad + \PROSPER      & \underline{58.78}&56.31 &79.27 & 67.14&62.18 &\underline{64.74} & & \\
\hline
\end{tabular}
\caption{\textbf{Model performance on instruction, knowledge, and commonsense benchmarks.}
The highest score is shown in \textbf{bold} and the second-highest is \underline{underlined}.
Across these benchmarks, post-training methods incur small performance drops; among them, \texttt{PROSPER} ranks first on Qwen2.5-3B-Instruct (by average score).}
\label{tab:sup_2}
\end{table}

Also, we provide the results on the checkpoint released by \cite{viswanathan2025checklists}.

\begin{table}[h]
\centering

\begin{minipage}{\columnwidth}
\centering
\renewcommand{\arraystretch}{1.2}
\resizebox{\columnwidth}{!}{
\begin{tabular}{lcccc}
\hline
\textbf{Model} 
& \multicolumn{2}{c}{\textbf{Arena-Hard}} 
& \multicolumn{2}{c}{\textbf{AlpacaEval}} \\
\hline
& \textbf{Vanilla} 
& \textbf{Style-Controlled}
& \textbf{Vanilla} 
& \textbf{Length-Controlled} \\
\hline
\textit{Qwen2.5-7B-Instruct} 
& 42.4 & 44.2 & 37.1 & 25.32 \\
\quad + RLCF (Released)     
& 42.2 & 44.6 & 42.3 & 28.08 \\
\hline
\end{tabular}
}
\end{minipage}

\vspace{0.8em}

\begin{minipage}{\columnwidth}
\centering
\renewcommand{\arraystretch}{1.25}
\resizebox{\columnwidth}{!}{
\begin{tabular}{lcccccc}
\hline
\textbf{Model} 
& \textbf{IFEval} 
& \textbf{MMLU} 
& \textbf{ARC} 
& \textbf{HellaSwag} 
& \textbf{TruthfulQA} 
& \textbf{Avg} \\
\hline
\textit{Qwen2.5-7B-Instruct} 
& 71.16 & 63.39 & 85.23 & 74.54 & 68.42 & 72.55 \\
\quad + RLCF (Released) 
& 67.83 & 63.16 & 85.41 & 74.51 & 67.44 & 71.67 \\
\hline
\end{tabular}
}
\end{minipage}

\end{table}

\section{Detailed Proof in Section~\ref{sec:theoretical}}\label{sec:detailed_proof}
To prove Theorem \ref{thm: formal_bound}, we break the proof into three different parts:
\begin{enumerate}
\item Concavity of the optimization objective. 
\item Subgradient computation and estimation.
\item Convergence result.
\end{enumerate}
\subsection{Concavity of the Optimization Objective}
Recall that the optimization objective is 
\[\max_{\pi\in\Pi}\ \min_{w:\Xs\rightarrow\Delta_{m(x)}}\ \min_{\pi'\in\Pi}\ 
\E_{x}\big[\langle w(x), \Ps(\pi\succ\pi'\mid x)\rangle +\ \beta\,\KL\!\big(\pi'( x)\,\|\,\pi_{\reff}( x)\big)\big].\]
Using the results from~\citet{ziebart2008maxent}, $\pi'$ has a closed-form solution, and we can then simplify the objective to
\[\max_{\pi\in\Pi}\ \min_{w:\Xs\rightarrow\Delta_{m(x)}}\ \E_x[-\beta\log Z(w,\pi|x)],\]
where \[
    Z(w,\pi\mid x)=\E_{y' \sim \pi_{\reff}(x)} [\exp(-\InAngles{w(x),\mathcal{P}_\pi(y'\mid x)}/\beta)] .
\] 
To further simplify the objective, we invoke the following lemma:
\begin{lemma}\label{lem: concave_Z}
$-\beta\log Z(w,\pi\mid x)$ is concave in $w(x)$ and $\pi$.
\end{lemma}

\begin{proof}
   By Holder's inequality, for any $w_1,w_2$ and $\theta\in[0,1]$,
   \begin{align*}
 Z(\theta w_1+(1-\theta)w_2,\pi\mid x)&=\E_{y'\sim\pi_{\reff}}[\exp(-\InAngles{w_1(x),\Ps_\pi(y'\mid x)})^\theta\cdot\exp(-\InAngles{w_2(x),\Ps_\pi(y'\mid x)})^{1-\theta}]\\
 &\le \E_{y'\sim\pi_{\reff}}[\exp(-\InAngles{w_1(x),\Ps_\pi(y'\mid x)})]^\theta\cdot \E_{y'\sim\pi_{\reff}}[\exp(-\InAngles{w_2(x),\Ps_\pi(y'\mid x)})]^{1-\theta}\\
 &= Z(w_1,\pi\mid x)^\theta Z(w_2,\pi\mid x)^{1-\theta}.
\end{align*}
 Taking $-\beta\log(\cdot)$ on both sides, we have 
 \[-\beta \log  Z(\theta w_1+(1-\theta)w_2,\pi\mid x) \geq \theta(-\beta \log Z(w_1,\pi\mid x)) + (1-\theta)(-\beta \log Z(w_2,\pi\mid x)).\]
 Thus, $-\beta\log Z(w,\pi\mid x)$ is concave in $w(x)$. A similar argument applies for $\pi$.
\end{proof}
By using a concave-form statement of Bauer’s maximum principle~\cite{bauer1958minimalstellen} as shown in theorem~\ref{thm:concave-min-at-extreme}, it suffices to enumerate only the unit vector in $\Delta_{m(x)}$ for each $w$.
\begin{theorem}[Bauer's maximum principle, concave form]\label{thm:concave-min-at-extreme}
Let $C$ be a nonempty compact convex subset of a (Hausdorff) locally convex topological vector space and let $\Cs$ be the set of extreme points in $C$. Let $f: C \to \mathbb{R}$ be concave and lower semicontinuous (in particular, continuous suffices).
Then $f$ attains its minimum at an extreme point of $C$:
\[
\min_{x\in C} f(x) \;=\; \min_{x\in \Cs} f(x).
\]
\end{theorem}

By \cref{thm:concave-min-at-extreme}, the optimization problem becomes $\max_{\pi \in \Pi} \mathbb{E}_{x}[-\beta\log Z^{k^{\star}(x)}(\pi|x)]$. Now we show 
\[\mathbb{E}_{x}\InBrackets{-\beta\log Z^{k^{\star}(x)}(\pi|x)}=\mathbb{E}_{x}\InBrackets{\min_{k\in [m(x)]}-\beta\log Z^{k}(\pi|x)}\] 
is concave in $\pi$. It suffices to show  $\min_{k\in [m(x)]}-\beta\log Z^{k}(\pi|x)$ is concave in $\pi$. This is because $-\beta\log Z^k$ is concave in $\pi$ and the minimum of $m(x)$ concave functions is a concave function. Then the question boils down to \textbf{(sub)gradient computation and estimation}.
\subsection{Subgradient computation and estimation}
By a classical result in convex analysis~\cite{nesterov2018lectures}, we have:
\begin{lemma}\label{lem:subgradient_max}
Given any integer $N$, Let \(f(x)=\max_{i\in[N]} f_i(x)\), where each \(f_i\) is convex. Then the subdifferential of \(f\) at \(x\) satisfies
\[
\partial f(x)=\mathrm{conv}\bigl\{\partial f_i(x): i\in \argmax_i f_i(x)\bigr\}.
\]
\end{lemma}

\begin{lemma}\label{lem:gradient_calculate}
\[
    g^k(z\mid\pi,x)\triangleq[\nabla_\pi-\beta\log Z^k(\pi\mid x)](z)=\frac{\E_{y'\sim\pi_{\reff}}[\Ps^k(z\succ y'\mid x)\exp(-\E_{y\sim\pi}[\Ps^k(y\succ y'\mid x)]/\beta)]}{\E_{y'\sim\pi_{\reff}}[\exp(-\E_{y\sim\pi}[\Ps^k(y\succ y'\mid x)]/\beta)]}.
\]
\end{lemma}
\begin{proof}
    By calculation.
\end{proof}
\begin{lemma}
    The gradient of the objective function $-\beta\log Z^{k^\star(x)}(\pi\mid x)$ is 
\[g(z\mid\pi,x)=\frac{\E_{y'\sim\pi_{\reff}}[\Ps^{k^{\star}(x)}(z\succ y'\mid x)\exp(-\E_{y\sim\pi}[\Ps^{k^{\star}(x)}(y\succ y'\mid x)]/\beta)]}{\E_{y'\sim\pi_{\reff}}[\exp(-\E_{y\sim\pi}[\Ps^{k^{\star}(x)}(y\succ y'\mid x)]/\beta)]}.\]
\end{lemma}
\begin{proof}
    Combining Lemma~\ref{lem:subgradient_max} and Lemma~\ref{lem:gradient_calculate} yields the conclusion.
\end{proof}
\subsection{Convergence Result}

For notational convenience, we write $\pi_t$ in place of $\pi_{\theta_t}$ and $g_t(x,z)$ in place of $g(z\mid \pi_{\theta_t}, x)$. In the following, we define $f_t(x,z):=\frac{1}{\eta}\log\frac{\pi_{t+1}(z\mid x)}{\pi_t(z\mid x)}$. 

\begin{lemma}\label{lem: smallsquare_pref}
Consider any $t \in [T]$. Define $\Delta(x, z) = f_t(x, z) - g_t(x, z)$. Define $\Delta_{\pi_t}(x) = \mathbb{E}_{z \sim \pi_t(x)} \Delta(x, z)$ and $\Delta_{\pi_{\reff}}(x) = \mathbb{E}_{z \sim \pi_{\reff}(x)} \Delta(x, z)$. Under \cref{ass: regression}, for all $t$, we have the following:
\begin{align}
\mathbb{E}_{x, z \sim \pi_t(x)} \left( f_t(x, z) - g_t(x, z) - \Delta_{\pi_t}(x) \right)^2 &\leq \epsilon, \tag{13} \\
\mathbb{E}_{x, z \sim \pi_{\reff}( x)} \left( f_t(x, z) - g_t(x, z) - \Delta_{\pi_{\reff}}(x) \right)^2 &\leq \epsilon, \tag{14} \\
\mathbb{E}_{x} \left( \Delta_{\pi_t}(x) - \Delta_{\pi_{\reff}}(x) \right)^2 &\leq \epsilon. \tag{15}
\end{align}
\end{lemma}

\begin{proof}
From \cref{ass: regression}, we have:
\begin{align*}
&\mathbb{E}_{x, z_1 \sim \pi_t, z_2 \sim \pi_{\reff}} \left( f_t(x, z_1) - \Delta_{\pi_t}(x) - g_t(x, z_1) - \left( f_t(x, z_2) - \Delta_{\pi_{\reff}}(x) - g_t(x, z_2) \right) + \Delta_{\pi_t}(x) - \Delta_{\pi_{\reff}}(x) \right)^2 \\
=& \mathbb{E}_{x, z_1 \sim \pi_t} \left( f_t(x, z_1) - \Delta_{\pi_t}(x) - g_t(x, z_1) \right)^2 
+ \mathbb{E}_{x, z_2 \sim \pi_{\reff}} \left( f_t(x, z_2) - \Delta_{\pi_{\reff}}(x) - g_t(x, z_2) \right)^2 \\
&\quad - 2 \mathbb{E}_{x, z_1 \sim \pi_t, z_2 \sim \pi_{\reff}} \left( f_t(x, z_1) - \Delta_{\pi_t}(x) - g_t(x, z_1) \right) \left( f_t(x, z_2) - \Delta_{\pi_{\reff}}(x) - g_t(x, z_2) \right) \\
&\quad + 2 \mathbb{E}_{x} \left( f_t(x, z_1) - \Delta_{\pi_t}(x) - g_t(x, z_1) \right) \left( \Delta_{\pi_t}(x) - \Delta_{\pi_{\reff}}(x) \right) \\
&\quad - 2 \mathbb{E}_{x, z_2 \sim \pi_{\reff}} \left( \Delta_{\pi_t}(x) - \Delta_{\pi_{\reff}}(x) \right) \left( f_t(x, z_2) - \Delta_{\pi_{\reff}}(x) - g_t(x, z_2) \right) 
+ \mathbb{E}_{x} \left( \Delta_{\pi_t}(x) - \Delta_{\pi_{\reff}}(x) \right)^2 \\
=& \mathbb{E}_{x, z_1 \sim \pi_t} \left( f_t(x, z_1) - \Delta_{\pi_t}(x) - g_t(x, z_1) \right)^2 
+ \mathbb{E}_{x, z_2 \sim \pi_{\reff}} \left( f_t(x, z_2) - \Delta_{\pi_{\reff}}(x) - g_t(x, z_2) \right)^2 \\
&\quad + \mathbb{E}_{x} \left( \Delta_{\pi_t}(x) - \Delta_{\pi_{\reff}}(x) \right)^2 \leq \epsilon.
\end{align*}
\end{proof}

By definition, we have $\Delta(x,z)=\frac{1}{\eta}\log\frac{\pi_{t+1}(z\mid x)}{\pi_t(z\mid x)}-g_t(x,z)$. Taking $\exp$ on both sides, we have
\[
\pi_{t+1}(z\mid x)=\pi_t(z\mid x)\exp(\eta(g_t(x,z)+\Delta(x,z)))=\frac{\pi_t(z\mid x)\exp(\eta(g_t(x,z)+\Delta(x,z)-\Delta_{\pi_{\reff}}(x)))}{\exp(-\eta\Delta_{\pi_{\reff}}(x))}.
\]

Denote $h_t(x,z):=g_t(x,z)+\Delta(x,z)-\Delta_{\pi_{\reff}(x)}$ and the advantage $A_t(x,z)=h_t(x,z)-\E_{z'\sim\pi_t(x)}[h_t(x,z')]$. We can rewrite the update rule as
\[
\pi_{t+1}(z\mid x)\propto\pi_t(z\mid x)\exp(\eta A_t(x,z)).
\]

\begin{lemma}\label{lem: MD}
Assume $\max_{x, y, t} |A_t(x, y)| \leq A \in \mathbb{R}^+$, and $\pi_0( x)$ is uniform over $\mathcal{Y}$. Then with
\[
\eta = \sqrt{\ln(|\mathcal{Y}|)/(A^2 T)},
\]
for the sequence of policies computed by \texttt{PROSPER}, we have:
\[
\forall \pi, x : \sum_{t=0}^{T-1} \mathbb{E}_{y \sim \pi_t(x)} A_t(x, y) \leq 2A \sqrt{\ln(|\mathcal{Y}|) T}.
\]
\end{lemma}

\begin{proof} Start with
\[
\pi_{t+1}(y \mid x) = \pi_t(y \mid x) \exp(\eta A_t(x, y)) / Z_t(x),
\]
where $Z_t(x)$ is the normalization constant. Taking log on both sides and add $\mathbb{E}_{y \sim \pi_t( x)}$, we have:
\[
- \KL(\pi(x) \| \pi_{t+1}(x)) = - \KL(\pi(x) \| \pi_t(x)) + \eta \mathbb{E}_{y \sim \pi(x)} A_t(x, y) - \mathbb{E}_{y \sim \pi(x)} \ln Z_t(x).
\]

Rearranging terms, we get:
\[
- \KL(\pi(x) \| \pi_t(x)) + \KL(\pi(x) \| \pi_{t+1}(x)) = \mathbb{E}_{y \sim \pi(x)}[- \eta A_t(x, y) + \ln Z_t(x)].
\]

For $\ln Z_t(x)$, using the condition that $\eta \leq 1/A$, we have $\eta A_t(x, y) \leq 1$, which allows us to use the inequality $\exp(x) \leq 1 + x + x^2$ for any $x \leq 1$, which leads to the following inequality:
\begin{align*}
\ln Z_t(x) &= \ln \left( \mathbb{E}_{y \sim \pi_t(x)} \exp(\eta A_t(x, y)) \right) \\
&\leq \ln \left( \sum_y \pi_t(y \mid x) \left( 1 + \eta A_t(x, y) + \eta^2 A_t(x, y)^2 \right) \right) \\
&\leq \ln \left( 1 + 0 + \eta^2 A^2 \right) \leq \eta^2 A^2,
\end{align*}
where the last inequality uses $\ln(1 + x) \leq x$, and we used the fact that $\mathbb{E}_{y \sim \pi_t(x)} A_t(x, y) = 0$ due to the definition of advantage $A_t$.

Thus, we have:
\[
- \KL(\pi(x) \| \pi_t(x)) + \KL(\pi(x) \| \pi_{t+1}(x)) \leq - \eta \mathbb{E}_{y \sim \pi(x)} [A_t(x, y)] + \eta^2 A^2.
\]

Sum over all iterations and do the telescoping sum, we get:
\[
\sum_{t=0}^{T-1} \mathbb{E}_{y \sim \pi(x)} A_t(x, y) \leq \KL(\pi(x) \| \pi_0(x))/\eta + T \eta A^2 \leq \ln(|\mathcal{Y}|)/\eta + T \eta A^2.
\]

With $\eta = \sqrt{\ln(|\mathcal{Y}|)/(A^2 T)}$, we conclude the proof.
\end{proof}
Now we can prove the Theorem~\ref{thm: formal_bound} formally:
\FormalBound*
\begin{proof}
Fix a comparator policy $\pi^\star$. We begin by considering the average performance gap between $\pi^\star$ and $\pi_t$ over $t=0,1,\ldots,T-1$:
\begin{align*}
\frac{1}{T} \sum_{t=0}^{T-1}\InParentheses{V(\pi^\star)-V(\pi_t)}&\le\frac{1}{T} \sum_{t=0}^{T-1}\InParentheses{\InAngles{\nabla V(\pi_t),\pi^\star-\pi_t}}\\
&=\frac{1}{T}\sum_{t=0}^{T-1}\InParentheses{\E_{x,y\sim\pi^\star(x)}[g_t(x,y)]-\E_{x,y\sim\pi_t(x)}[g_t(x,y)]}\\
&=\frac{1}{T} \sum_{t=0}^{T-1} \mathbb{E}_{x,y \sim \pi^\star(x)} \left( A^{\pi_t}(x, y) \right),
\end{align*}
where the first inequality holds because of concavity and we define the real advantage $A^{\pi_t}(x, y) := g_t(x, y) - \mathbb{E}_{z \sim \pi_t(x)} g_t(x, z)$. Continuing, we have:
\[
\frac{1}{T} \sum_{t=0}^{T-1} \mathbb{E}_{x,y \sim \pi^\star(x)} \left( A^{\pi_t}(x, y) \right)
= \frac{1}{T} \sum_{t=0}^{T-1} \mathbb{E}_{x,y \sim \pi^\star(x)} \left( A_t(x, y) \right)
+ \frac{1}{T} \sum_{t=0}^{T-1} \mathbb{E}_{x,y \sim \pi^\star(x)} \left( A^{\pi_t}(x, y) - A_t(x, y) \right)
\]

\[
\leq 2A \sqrt{\frac{\ln(|\mathcal{Y}|)}{T}} + \frac{1}{T} \sum_{t=0}^{T-1} \sqrt{\mathbb{E}_{x,y \sim \pi^\star(x)} \left( A^{\pi_t}(x, y) - A_t(x, y) \right)^2},
\]
where the last inequality uses \cref{lem: MD}. We now just need to bound $\mathbb{E}_{x,y \sim \pi^\star(x)} \left( A^{\pi_t}(x, y) - A_t(x, y) \right)^2$.

\[
\mathbb{E}_{x,y \sim \pi^\star(x)} \left( A^{\pi_t}(x, y) - A_t(x, y) \right)^2 
= \mathbb{E}_{x} \mathbb{E}_{y \sim \pi_{\reff}(x)} \left( \frac{\pi^\star(y \mid x)}{\pi_{\reff}(y \mid x)} (A^{\pi_t}(x, y) - A_t(x, y))^2 \right)
\]

\[
\leq C_{\pi_{\reff}\rightarrow\pi^\star} \cdot \mathbb{E}_{x, y \sim \pi_{\reff}(x)} \left( A^{\pi_t}(x, y) - A_t(x, y) \right)^2
\]

where the last inequality uses the definition of \textit{concentrability coefficient} $C_{\pi_{\reff}\rightarrow\pi^\star}$, which is \[C_{\pi_{\reff}\rightarrow\pi}=\max_{x,y}\frac{\pi(y\mid x)}{\pi_{\reff}(y\mid x)}.\]

We now bound $\mathbb{E}_{x,y \sim \pi_{\reff}( x)} \left( A^{\pi_t}(x, y) - A_t(x, y) \right)^2$. Recall the definition of $A_t$ from \cref{lem: MD}:

\begin{align*}
\mathbb{E}_{x,y \sim \pi_{\reff}( x)} \left( A^{\pi_t}(x, y) - A_t(x, y) \right)^2
&= \mathbb{E}_{x, y, y' \sim \pi_t(x)} \left( \hat{g}_t(x, y) - \mathbb{E}_{y' \sim \pi_t(x)} \hat{g}_t(x, y') - h_t(x, y) + \mathbb{E}_{y' \sim \pi_t(x)} h_t(x, y') \right)^2 \\
&\leq 2 \mathbb{E}_{x, y \sim \pi_{\reff}(x)} \left( \hat{g}_t(x, y) - h_t(x, y) \right)^2
+ 2 \mathbb{E}_{x, y' \sim \pi_t(x)} \left( \mathbb{E}_{y \sim \pi_t(x)} [\hat{g}_t(x, y) - h_t(x, y)] \right)^2
\end{align*}

Recall that $h_t(x, y) = g(x, y) + \Delta(x, y) - \Delta_{\pi_{\reff}}(x)$, and from \cref{lem: smallsquare_pref}, we can see that:
\[
\mathbb{E}_{x, y \sim \pi_{\reff}(x)} \left( g_t(x, y) - h_t(x, y) \right)^2 
= \mathbb{E}_{x, y \sim \pi_{\reff}(x)} \left( \Delta(x, y) - \Delta_{\pi_{\reff}}(x) \right)^2 \leq \epsilon.
\]

For 
\[
\mathbb{E}_{x} \mathbb{E}_{y' \sim \pi_t( x)} \left( g_t(x, y') - h_t(x, y') \right)^2,
\]
we have:
\begin{align*}
\mathbb{E}_{x} \mathbb{E}_{y' \sim \pi_t( x)} \left( \hat{g}_t(x, y') - h_t(x, y') \right)^2 
&= \mathbb{E}_{x} \mathbb{E}_{y' \sim \pi_t( x)} \left( \Delta_{\pi_t}(x, y') - \Delta_{\pi_{\reff}}(x) \right)^2 \\
&= \mathbb{E}_{x} \mathbb{E}_{y' \sim \pi_t( x)} \left( \Delta(x, y') - \Delta_{\pi_t}(x) + \Delta_{\pi_t}(x) - \Delta_{\pi_{\reff}}(x) \right)^2 \\
&\leq 2 \mathbb{E}_{x} \mathbb{E}_{y' \sim \pi_t( x)} \left( \Delta(x, y') - \Delta_{\pi_t}(x) \right)^2 
+ 2 \mathbb{E}_{x} \left( \Delta_{\pi_t}(x) - \Delta_{\pi_{\reff}}(x) \right)^2 \\
&\leq 4\epsilon,
\end{align*}
where the last inequality uses \cref{lem: smallsquare_pref} again. This step relies on the fact that one of the samples is always on-policy, i.e., from $\pi_t$.

Combine things together, we can conclude that:
\[
\mathbb{E}_{x} \mathbb{E}_{y \sim \pi^\star( x)} \left( A^{\pi_t}(x, y) - A_t(x, y) \right)^2 \leq C_{\pi_{\reff}\rightarrow\pi^\star}(10\epsilon).
\]

We can conclude:
\begin{align*}
\frac{1}{T} \sum_{t=0}^{T-1}\InParentheses{V(\pi^\star)-V(\pi_t)}
&\leq 2A \sqrt{ \frac{\ln |\mathcal{Y}|}{T} } + \frac{1}{T} \sum_{t} \sqrt{C_{\pi_{\reff}\rightarrow\pi^\star} 10\epsilon} \\
&= 2A \sqrt{ \frac{\ln |\mathcal{Y}|}{T} } + \sqrt{C_{\pi_{\reff}\rightarrow\pi^\star} 10\epsilon}.
\end{align*}
\end{proof}

\newpage
\section{Implementation Details}\label{sec:Implementation_details}

\subsection{Implementation Details and Hyperparameters}
\label{app:hyperparameters}

Our algorithm follows an off-policy training paradigm, where each epoch is partitioned into three distinct phases: response generation, preference scoring, and policy optimization. We execute two training epochs per model. The following sections detail the hyperparameter configurations and computational requirements for each phase.

\noindent \textbf{Response Generation.}
In each epoch, we sample $K=8$ responses per prompt from the current policy. The sampling parameters and the computational resources utilized for this stage are summarized in Table~\ref{tab:response_generation_hparams}.

\begin{table}[H]
\centering
\small
\caption{Hyperparameters and compute configuration for response generation.}
\label{tab:response_generation_hparams}
\begin{tabular}{ll}
\toprule
\textbf{Parameter} & \textbf{Value} \\
\midrule
Temperature & 0.8 \\
Top-$p$ & 0.9 \\
Max new tokens & 2048 \\
\midrule
Hardware & $1 \times (16 \times \text{NVIDIA H100 80GB})$ \\
Total compute time & $\approx 2$ hours per model per epoch \\
\bottomrule
\end{tabular}
\end{table}

\noindent \textbf{Score Generation.}
We utilize an LLM-as-a-judge framework to evaluate the generated response pairs. Each pair is integrated into a structured evaluation template, and any resulting sequences exceeding a length of 4096 tokens are excluded to maintain consistency. To ensure the robustness of the preference signals, we sample 10 independent judgments per (prompt, response pair). To mitigate inherent positional bias, we implement a symmetric evaluation protocol: the presentation order of the responses is permuted for half of the samples (5 out of 10), and the corresponding scores are reversed before computing the final average. Detailed scoring parameters are summarized in Table~\ref{tab:score_generation_hparams}.

\begin{table}[H]
\centering
\small
\caption{Hyperparameters and compute configuration for score generation.}
\label{tab:score_generation_hparams}
\begin{tabular}{ll}
\toprule
\textbf{Parameter} & \textbf{Value} \\
\midrule
Temperature & 0.8 \\
Top-$p$ & 0.9 \\
Top-$k$ & 20 \\
Max new tokens & 256 \\
Selection pairs & 4 \\
Base pairs & 2 \\
Current pairs & 2 \\
\midrule
Hardware & $1 \times (16 \times \text{NVIDIA H100 80GB})$ \\
Total compute time & $\approx 50$ hours per model per epoch \\
\bottomrule
\end{tabular}
\end{table}

\noindent \textbf{Training and Optimization.}
Following the derivation of game-theoretic training signals from the preference scores, we filter the dataset to retain the most informative signals. Specifically, we expand the dataset by constructing all possible pairwise comparisons within each prompt and then apply a global filter based on the score gap. We use a filtration ratio of $0.15$ for the first epoch and $0.17$ for the second, selected based on performance on the \textsc{Arena-Hard} benchmark. The policy is then optimized using the hyperparameters listed in Table~\ref{tab:training_hparams}.

\begin{table}[H]
\centering
\small
\caption{Hyperparameters and compute configuration for policy training.}
\label{tab:training_hparams}
\begin{tabular}{ll}
\toprule
\textbf{Parameter} & \textbf{Value} \\
\midrule
Batch size & 128 \\
Max input length & 1024 \\
Max sequence length & 2048 \\
Optimizer & AdamW \\
AdamW $\epsilon$ & $1 \times 10^{-8}$ \\
Learning rate & $3 \times 10^{-7}$ \\
Weight decay & $1 \times 10^{-6}$ \\
Warmup ratio & 0.1 \\
Max gradient norm & 1.0 \\
Seed & 555134 \\
\midrule
Hardware & $1 \times (8 \times \text{NVIDIA H100 80GB})$ \\
Total compute time & $\approx 3$ hours per model per epoch \\
\bottomrule
\end{tabular}
\end{table}

\subsection{Prompts for LLM Judge}
\begin{figure}[!t]
    \centering
    \includegraphics[width=\linewidth,height=0.9\textheight,keepaspectratio]{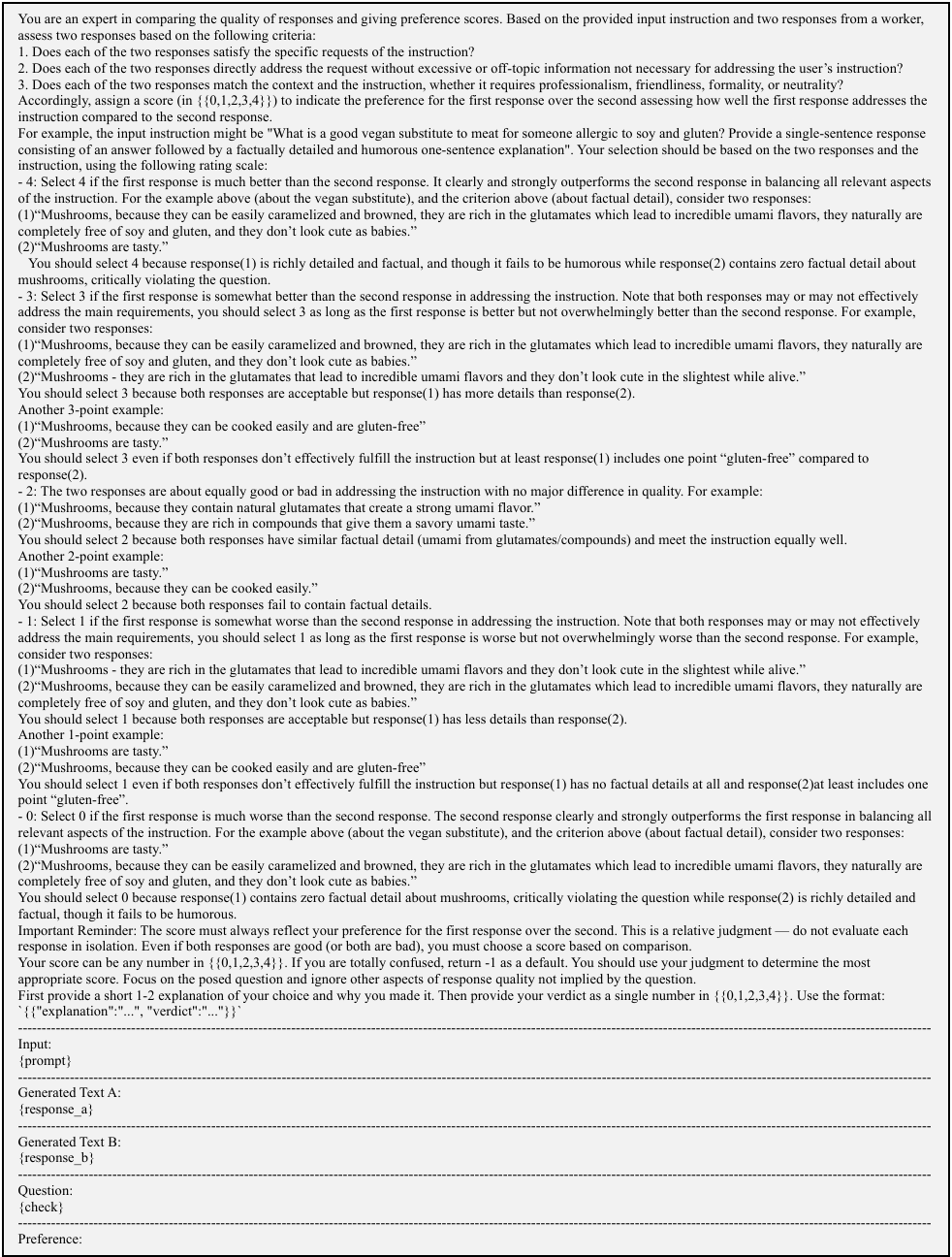}
    \caption{Prompt for generating preference score according to specific criteria.}
    \label{fig:multi_prompt}
\end{figure}
\begin{figure}[!t]
    \centering
    \includegraphics[width=\linewidth,height=0.9\textheight,keepaspectratio]{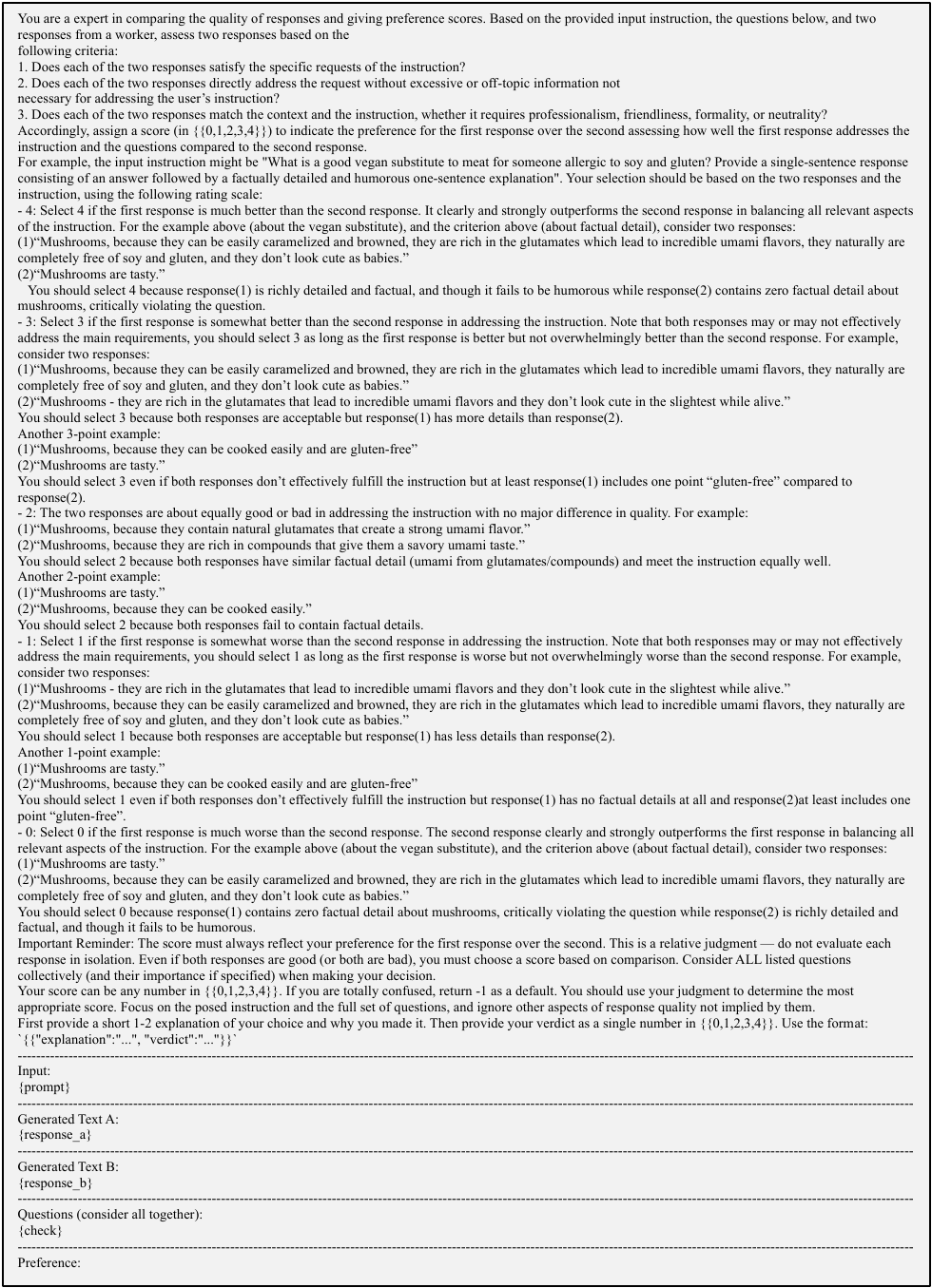}
    \caption{Prompt for generating preference score according to multiple criterion.}
    \label{fig:fullcheck_prompt}
\end{figure}
\par We provide the prompts used for requirement checking in~\cref{fig:multi_prompt} and~\cref{fig:fullcheck_prompt}. In Figure~\ref{fig:multi_prompt}, each check for a given prompt is provided separately, whereas in Figure~\ref{fig:fullcheck_prompt}, all checks for the prompt are provided together. The templates are adapted from those in \citet{viswanathan2025checklists}, modified for pairwise comparison and augmented with explanation prompting to improve scoring quality.

\end{document}